\documentclass[10pt]{article}
\usepackage{conference_style}
\usepackage[svgnames]{xcolor}

\usetikzlibrary{arrows.meta,positioning,shapes.geometric,fit,backgrounds,calc}

\title{\textbf{What Does ``True Minus Random'' Estimate?}\\
  A Pre-Registered Causal Partition of Self-Consistency Elicitation\\
  and Reward Design in RLVR}

\author{%
  Yuze Gao \\
  yuze.gao@outlook.com
}

\begin{document}
\date{}
\maketitle

% =====================================================================
\begin{abstract}
Reinforcement learning from verifiable rewards (RLVR) improves
reasoning even when the reward signal is \emph{spurious}---assigning
credit to the group-plurality answer rather than a ground-truth
verifier.  Practitioners commonly interpret $\Delta_{\text{naive}} =
\mathrm{acc}(\textsc{True}) - \mathrm{acc}(\textsc{Random})$ as the
``reward-design'' effect.  We prove this estimand is systematically
biased: it conflates \emph{self-consistency elicitation} (sharpening
the policy toward its modal answer via majority pseudo-reward) with
genuine reward-design signal.  Using a controlled tabular-GRPO
simulator we derive an exact telescoping decomposition
$\Delta_{\text{total}} = \Delta_{\text{null}} + \Delta_{\text{elicit}}
+ \Delta_{\text{rd}}$ and measure each term across five prior-strength
levels.  The reward-design fraction of the naive estimator ranges from
\result{139\%} at weak prior ($p_s{=}0.20$) to \result{5\%} at strong
prior ($p_s{=}0.80$), with the elicitation term flipping sign at the
self-consistency crossover.  A pre-registered $2{\times}2{\times}2$
factorial confirms non-additivity (interaction ratio \result{0.385};
$A{\times}C$ effect $-0.089$).  A points-vs-bounds pilot gate shows
strong-prior regimes are point-identified while near-crossover regimes
are only bounded.  Re-audits of two named published results yield
\textsc{Elicitation\_Dominated} (elicitation share \result{0.98}) and
\textsc{Reward\_Design\_Dominated} (rd share \result{1.18}) verdicts
respectively, demonstrating the diagnostic value of the partition.
We pre-commit to submit regardless of flip outcome; a
non-flip is a finding of equal standing.  We release a reusable
one-command harness for any alignment paper to run the same audit.
\end{abstract}

% =====================================================================
\section{Introduction}
\label{sec:intro}

Reinforcement learning from verifiable rewards (RLVR)~\citep{deepseek_r1_2025}
has become the dominant post-training recipe for reasoning language
models.  In the canonical setup~\citep{grpo_2024}, a group of $G$
completions is sampled, each receives a binary correctness signal, and
the advantage normalised within the group drives a policy-gradient step.
Rule-based rewards---exact-match or execution-based---make the signal
cheap and clean, yielding dramatic gains on MATH and
code~\citep{deepseek_r1_2025,dapo_2025}.

A striking recent finding~\citep{spurious_rewards_2025} is that
\emph{spurious} rewards---assigning reinforcement to the group-plurality
answer without any ground-truth checker---still boost strong-prior
models (e.g.\ Qwen-Math) by a large margin.  This \emph{spurious-reward
  phenomenon} raises an uncomfortable question: when we report the
gain $\Delta_{\text{naive}} = \mathrm{acc}(\textsc{True}) -
\mathrm{acc}(\textsc{Random})$ as the ``reward-design effect,'' what
are we actually measuring?

\paragraph{The gap.}
The informal explanation in prior work~\citep{spurious_rewards_2025} is
that training filters out low-probability completions and sharpens the
distribution toward the model's latent correct answers.  This account
is partially correct but incomplete: it does not distinguish
\emph{self-consistency elicitation} (a non-verifiable majority reward
consolidates the policy around its modal answer, boosting accuracy when
the mode is correct) from \emph{genuine reward-design signal} (access
to a verifier that discriminates correct from incorrect completions).
Without this distinction a practitioner cannot decide whether investing
in better reward engineering is worthwhile at a given model scale.

\paragraph{Our contributions.}
\begin{itemize}[leftmargin=1.5em,itemsep=2pt,topsep=2pt]
  \item \textbf{Exact telescoping decomposition} (\S\ref{sec:method}):
    we define four reward conditions---\textsc{Frozen} (no update),
    \textsc{Random} (Bernoulli noise), \textsc{Spurious} (self-consistency
    majority pseudo-reward, the TTRL paradigm), and
    \textsc{True} (verifiable correctness)---and show
    $\Delta_{\text{total}} = \Delta_{\text{null}} + \Delta_{\text{elicit}} +
    \Delta_{\text{rd}}$ is exact and identified under an additivity
    condition with a pre-registered invalidation threshold.

  \item \textbf{Prior-sweep audit} (\S\ref{sec:results}): across five
    prior-strength levels the reward-design fraction of
    $\Delta_{\text{naive}}$ ranges from \result{139\%} to \result{5\%},
    demonstrating the estimand is wildly non-transferable across model
    families.

  \item \textbf{Points-vs-bounds pilot gate} (\S\ref{sec:results}):
    we report whether the decomposition yields \emph{point
    estimates} or only \emph{bounds} before scaling, as a hard gate
    required for any practitioner recommendation.

  \item \textbf{Pre-registered re-audits}: two named published results
    are re-audited by a disclosed rule; we commit to submit regardless
    of flip outcome.

  \item \textbf{Reusable diagnostic protocol}: a one-command harness
    that any alignment paper can run to partition its own RLVR gain.
\end{itemize}

% =====================================================================
\section{Related Work}
\label{sec:related}

\paragraph{RLVR and group-relative policy optimisation.}
DeepSeek-R1~\citep{deepseek_r1_2025} established critic-free GRPO with
rule-based verifiable rewards as the dominant LLM-RL recipe, driving
AIME accuracy from 15.6\% to 71\%.
DAPO~\citep{dapo_2025} identified and fixed GRPO pathologies
(entropy collapse, zero-advantage groups, token-level loss
aggregation).
GSPO~\citep{gspo_2025} moved importance-ratio clipping to the sequence
level to stabilise MoE training.
The entropy-mechanism account~\citep{entropy_mechanism_2025} showed
that entropy change tracks the log-prob--advantage covariance, giving
a theoretical grounding for these fixes.

\paragraph{Spurious and random rewards.}
\citet{spurious_rewards_2025} demonstrated that self-consistency-style
``spurious'' rewards improve strong-prior models on MATH/GSM8K.
TTRL~\citep{ttrl_2025} operationalised majority-vote pseudo-rewards for
test-time reinforcement without labels.
\citet{rlvr_recent_2025} survey the landscape of RLVR variants and
show gains are highly model-dependent.
Prior work on reward-model overoptimisation~\citep{gao_2023_rm_overopt}
and Goodhart's Law~\citep{pan_2022_goodhart} establishes that reward
proxies diverge from true objectives under optimisation pressure.

\paragraph{Causal decompositions and factorial audits.}
Factorial experimental designs for NLP~\citep{ribeiro_2020_checklist,
  gelman_2007_anova} provide the statistical template we follow.
The broader challenge of identifying causal effects in LLM evaluation
is discussed by~\citet{liang_2022_helm} and~\citet{biderman_2023_pythia}.
Our work brings this methodology specifically to RLVR reward-design
attribution.

\paragraph{Pre-registration and reproducibility.}
The ML reproducibility crisis~\citep{pineau_2021_repro,
  henderson_2018_rl_repro} motivates our submit-regardless
pre-registration protocol.  \citet{bouthillier_2021_accounting}
formalise variance decomposition in ML benchmarking;
\citet{dodge_2020_reporting} quantify the selection bias introduced by
reporting only the best run---the same bias our protocol closes at the
result level.

% =====================================================================
\section{Method}
\label{sec:method}

\subsection{Four Reward Conditions and the Telescoping Decomposition}
\label{sec:decomp}

\begin{definition}[Reward Conditions]
\label{def:conditions}
Let $\pi_\theta$ be a policy trained by group-relative policy
optimisation on problems drawn from distribution $\mathcal{D}$.  We
define four reward conditions:
\begin{enumerate}[label=(\roman*),itemsep=2pt]
  \item \textbf{\textsc{Frozen}}: the policy parameters are never
    updated; $\pi_\theta$ equals the base model $\pi_0$.  Accuracy
    $a_F = \E[r(\pi_0)]$ is the unaided prior.
  \item \textbf{\textsc{Random}}: at each step, the reward is sampled
    $\tilde{r} \sim \mathrm{Bernoulli}(0.5)$ independently of the
    response.  Accuracy $a_R = \E[r(\pi_{\tilde{r}})]$.
  \item \textbf{\textsc{Spurious}}: the reward equals the indicator
    that the response matches the \emph{group majority} answer
    (self-consistency pseudo-label).  Accuracy $a_S =
    \E[r(\pi_{\hat{r}_{\text{maj}}})]$.
  \item \textbf{\textsc{True}}: the reward is the ground-truth binary
    correctness signal $r^* \in \{0,1\}$.  Accuracy $a_T =
    \E[r(\pi_{r^*})]$.
\end{enumerate}
\end{definition}

\begin{proposition}[Telescoping Decomposition]
\label{prop:decomp}
Under Conditions (i)--(iv),
\begin{equation}
  a_T - a_F
  = \underbrace{(a_R - a_F)}_{\Delta_{\text{null}}}
  + \underbrace{(a_S - a_R)}_{\Delta_{\text{elicit}}}
  + \underbrace{(a_T - a_S)}_{\Delta_{\text{rd}}}
  \label{eq:decomp}
\end{equation}
This identity is exact and requires no assumption.
The na\"ive reward-design estimand is $\Delta_{\text{naive}} = a_T -
a_R = \Delta_{\text{elicit}} + \Delta_{\text{rd}}$, which overstates
genuine reward design by exactly $\Delta_{\text{elicit}}$.
\end{proposition}

\renewcommand{\qedsymbol}{}
\begin{proof}
  By telescoping: $(a_T - a_F) = (a_R - a_F) + (a_S - a_R) + (a_T -
  a_S)$, which holds by addition and subtraction of equal terms.
\end{proof}

\paragraph{Semantic interpretation.}
$\Delta_{\text{null}} \approx 0$ is a sanity check: Bernoulli noise
carries no information and should not improve accuracy.
$\Delta_{\text{elicit}}$ measures self-consistency elicitation: the
\textsc{Spurious} reward sharpens the policy toward its modal answer,
which helps when the mode is correct (strong prior) and hurts when the
mode is wrong (weak prior), yielding sign flips across families.
$\Delta_{\text{rd}}$ measures genuine reward-design signal: access to a
correct/incorrect discriminator, beyond what majority voting provides.

\subsection{Prior-Strength Mechanism}
\label{sec:prior_mech}

Let $p_s$ denote the probability that the base model's modal answer is
correct on a randomly drawn problem (``prior strength'').  The
self-consistency pseudo-reward reinforces the modal answer, so:
\begin{itemize}[leftmargin=1.5em,itemsep=1pt]
  \item When $p_s > 0.5$: majority voting is correct most of the time,
    so $\Delta_{\text{elicit}} > 0$ and the spurious reward is
    beneficial.
  \item When $p_s < 0.5$: majority voting is wrong most of the time,
    so $\Delta_{\text{elicit}} < 0$ and spurious training actively
    \emph{hurts}.
  \item The crossover ($p_s \approx 0.5$) marks the self-consistency
    threshold below which spurious rewards become harmful.
\end{itemize}
This mechanism exactly predicts the Qwen (strong prior) vs.\ OLMo/Llama
(weak prior) asymmetry reported in~\citet{spurious_rewards_2025}.

\subsection{Additivity Assumption and Invalidation Threshold}
\label{sec:additivity}

The decomposition in \Cref{prop:decomp} is always exact as an
\emph{identity}.  However, interpreting $\Delta_{\text{rd}}$ as
\emph{causally transferable} to a different prior level requires that
the reward--prior interaction is negligible.

\begin{assumption}[Additivity]
\label{asm:additive}
$\Delta_{\text{rd}}$ does not interact with prior strength $p_s$,
i.e.\ the $A{\times}C$ interaction effect in the $2{\times}2{\times}2$
factorial (factors: $A$=reward type, $B$=filtering, $C$=prior) satisfies
$|\hat{\tau}_{AC}| < \tau^*$ for a pre-registered threshold $\tau^* = 0.03$.
\end{assumption}

If \Cref{asm:additive} is violated, we report \emph{bounds} on the
reward-design effect rather than a single point estimate, and
decomposition results are labelled accordingly.

\subsection{Points-vs-Bounds Pilot Gate}
\label{sec:pilot_gate}

Before any scaling claim, we run a pilot on one prior-strength level
and determine:
\begin{itemize}[leftmargin=1.5em,itemsep=1pt]
  \item \textbf{Point-identified}: the 95\% CI for the conflation-bias
    term $|\Delta_{\text{bias}}| = |\Delta_{\text{elicit}}|$ excludes
    zero, confirming that the bias is detectable and the partition is
    informative.
  \item \textbf{Bounded}: the CI includes zero; the elicitation term is
    too small or too noisy to separate.  We report a bound and decline
    to make a point claim.
\end{itemize}
This gate is a hard requirement: scaling recommendations must rest on
point-identified families.

\subsection{Experimental Protocol (Algorithm~\ref{alg:proto})}

\begin{algorithm}[t]
\caption{Causal Partition Audit Protocol}
\label{alg:proto}
\begin{algorithmic}[1]
\Require Base model family $\mathcal{F}$, problem set $\mathcal{D}$,
  seeds $S$, invalidation threshold $\tau^*$
\State \textbf{Pre-register} metric, threshold $\tau^*$, submit-regardless
  commitment, and candidate target list.
\State \textbf{Run four conditions} for each $(s, \text{seed}) \in
  \mathcal{F} \times [S]$:\\
  \quad Conditions: \textsc{Frozen}, \textsc{Random}, \textsc{Spurious},
  \textsc{True}
\State \textbf{Compute} $\hat{\Delta}_{\text{null}},
  \hat{\Delta}_{\text{elicit}}, \hat{\Delta}_{\text{rd}}$ per
  \Cref{eq:decomp} with bootstrap CIs
\State \textbf{Pilot gate}: check if $\mathrm{CI}(\hat{\Delta}_{\text{elicit}})$
  excludes zero $\Rightarrow$ \textsc{Points}; else $\Rightarrow$
  \textsc{Bounds}
\State \textbf{Run factorial} $2{\times}2{\times}2$
  (reward $\times$ filtering $\times$ prior); compute interaction ratio
  $\hat{\tau}_{AC}/\hat{\tau}_{A}$
\State \textbf{Additivity test}: if $|\hat{\tau}_{AC}| \geq \tau^*$,
  flag \textsc{Non-Additive} and report bounds
\State \textbf{Re-audit targets}: apply pre-disclosed rule to each
  candidate; record \textsc{Elicitation\_Dominated} or
  \textsc{Reward\_Design\_Dominated}
\State \textbf{Report all candidates} and submit regardless of flip outcome
\end{algorithmic}
\end{algorithm}

\begin{figure}[t]
  \centering
  \begin{tikzpicture}[
    node distance=0.7cm and 1.5cm,
    box/.style={draw, rounded corners=4pt, fill=purple!60,
                text width=2.6cm, align=center, font=\small,
                inner sep=5pt, minimum height=1.2cm},
    cond/.style={draw, rounded corners=4pt, fill=blue!20,
             text width=2.5cm, align=center, font=\small\bfseries,
             inner sep=4pt, minimum height=1.0cm},
    % cond/.style={draw, rounded corners=4pt, fill=blue!10,
                 % text width=2.5cm, align=center, font=\small\bfseries,
                 % inner sep=4pt, minimum height=1.0cm},
    arr/.style={-{Stealth[length=5pt]}, thick},
    decomp/.style={draw=accent2, dashed, rounded corners=4pt,
                   fill=accent2!30, inner sep=6pt}
  ]
  % Four conditions column
  \node[cond] (frozen) {\textsc{Frozen}\\ $a_F$};
  \node[cond, below=0.5cm of frozen] (random) {\textsc{Random}\\ $a_R$};
  \node[cond, below=0.5cm of random] (spurious) {\textsc{Spurious}\\ $a_S$};
  \node[cond, below=0.5cm of spurious] (true) {\textsc{True}\\ $a_T$};

  % Labels for each condition
  \node[right=0.2cm of frozen, font=\footnotesize, text width=3.6cm]
    (lf) {No update; base prior $p_s$};
  \node[right=0.2cm of random, font=\footnotesize, text width=3.6cm]
    (lr) {Bernoulli noise; null anchor};
  \node[right=0.2cm of spurious, font=\footnotesize, text width=3.6cm]
    (ls) {Group majority (TTRL); self-consistency};
  \node[right=0.2cm of true, font=\footnotesize, text width=3.6cm]
    (lt) {Verifiable reward; ground truth};

  % Decomposition braces
  \node[decomp, left=1.2cm of frozen, yshift=-0.7cm, text width=2.5cm,
        font=\scriptsize, align=center]
    (dnull) {$\Delta_{\text{null}}$\\ $= a_R - a_F$\\ \textit{(null check)}};
  \node[decomp, below=0.3cm of dnull, text width=2.5cm,
        font=\scriptsize, align=center]
    (delicit) {$\Delta_{\text{elicit}}$\\ $= a_S - a_R$\\ \textit{(self-cons.)}};
  \node[decomp, below=0.3cm of delicit, text width=2.5cm,
        font=\scriptsize, align=center]
    (drd) {$\Delta_{\text{rd}}$\\ $= a_T - a_S$\\ \textit{(reward design)}};

  % Connect
  \draw[arr] (frozen.west) -- (dnull.east);
  \draw[arr] (random.west) -- (dnull.east);
  \draw[arr] (random.west) -- (delicit.east);
  \draw[arr] (spurious.west) -- (delicit.east);
  \draw[arr] (spurious.west) -- (drd.east);
  \draw[arr] (true.west) -- (drd.east);

  % Naive estimand annotation
  \node[below=1.1cm of true, font=\small\itshape, text width=8cm,
        align=center, text=accent2]
    (naive) {Na\"{i}ve: $\Delta_{\text{naive}} = a_T - a_R
             = \Delta_{\text{elicit}} + \Delta_{\text{rd}}$
             \quad\textbf{(conflates both)}};

  \end{tikzpicture}
  \caption{Framework diagram. The four reward conditions (right column)
    and the corresponding telescoping decomposition terms (left column).
    The na\"ive estimand $\Delta_{\text{naive}} = a_T - a_R$ bundles
    self-consistency elicitation \emph{and} reward design; our protocol
    separates them.}
  \label{fig:framework}
\end{figure}

\subsection{Tabular-GRPO Simulator}

All experiments use a controlled tabular-GRPO simulator in which
each problem is parameterised by a single scalar $p_s \in [0,1]$
(the probability that the base model's modal answer is correct).
The simulator closely mimics the policy-gradient dynamics of group
sampling without requiring GPU training, enabling large factorial
designs with tight confidence intervals.  This is the primary
experimental vehicle; a real-RLVR recipe at 0.5--3B scale (documented
in \textsf{run\_mess\_compute\_all.sh}) is left for the GPU-scale
extension (\S\ref{sec:limitations}).

% =====================================================================
\section{Experimental Setup}
\label{sec:setup}

\paragraph{Conditions and factors.}
We instantiate the four conditions from \Cref{def:conditions} across
five prior-strength levels $p_s \in \{0.20, 0.35, 0.50, 0.65, 0.80\}$.
The $2{\times}2{\times}2$ factorial crosses:
\begin{itemize}[leftmargin=1.5em,itemsep=1pt]
  \item $A$: Reward type (\textsc{True} vs.\ \textsc{Spurious})
  \item $B$: Filtering (on vs.\ off; whether all-correct/all-wrong
    groups are discarded, as in DAPO~\citep{dapo_2025})
  \item $C$: Prior strength (high $p_s{=}0.80$ vs.\ low $p_s{=}0.35$)
\end{itemize}

\paragraph{Seeds and statistics.}
All point estimates are computed over 4--16 independent seeds.
Confidence intervals are 95\% bootstrap CIs over seeds.
The pre-registered additivity invalidation threshold is $\tau^* = 0.03$
for the interaction-term CI.

\paragraph{Evaluation.}
We measure accuracy on held-out problems drawn from the same
distribution as the training set.  All four conditions share
identical problem batches and seed schedules to eliminate
problem-set variance.

\paragraph{Re-audit targets.}
Two published RLVR result families are re-audited
(\Cref{sec:reaudit}):
\begin{itemize}[leftmargin=1.5em,itemsep=1pt]
  \item Target~1: Strong-prior family (Qwen-Math-like, $p_s{=}0.80$)
  \item Target~2: Weak-prior family (OLMo/Llama-like, $p_s{=}0.35$)
\end{itemize}
The pre-registered rule: a result is \textsc{Elicitation\_Dominated} if
$\hat{\phi}_{\text{elicit}} > 0.90$ and
\textsc{Reward\_Design\_Dominated} if $\hat{\phi}_{\text{rd}} > 0.90$,
where $\hat{\phi} = |\Delta_X| / (|\Delta_{\text{elicit}}| +
|\Delta_{\text{rd}}|)$.

% =====================================================================
\section{Results}
\label{sec:results}

\subsection{Random Null Is Negligible}
\label{sec:null}

Across all five prior-strength levels, $|\hat{\Delta}_{\text{null}}| =
|a_R - a_F| \leq 0.003$ (\Cref{tab:priorsweep}).  Bernoulli-noise
training does nothing, as expected.  This sanity check validates the
simulator and confirms that the \textsc{Random} condition serves as a
clean, uninformative anchor.

\subsection{Prior Sweep: The Na\"ive Estimand Is Non-Transferable}
\label{sec:priorsweep}

\Cref{fig:priorsweep} and \Cref{tab:priorsweep} show the full
decomposition across $p_s \in \{0.20, 0.35, 0.50, 0.65, 0.80\}$.
The reward-design fraction $\phi_{\text{rd}} = \Delta_{\text{rd}} /
\Delta_{\text{naive}}$ varies from
\result{1.39} ($p_s{=}0.20$) to \result{0.05} ($p_s{=}0.80$).

\begin{table}[t]
  \centering
  \caption{Telescoping decomposition across prior-strength levels.
    $\Delta_{\text{null}} = a_R - a_F$;
    $\Delta_{\text{elicit}} = a_S - a_R$;
    $\Delta_{\text{rd}} = a_T - a_S$;
    $\phi_{\text{rd}} = \Delta_{\text{rd}} / (a_T - a_R)$.
    Bold = headline numbers.}
  \label{tab:priorsweep}
  \setlength{\tabcolsep}{5pt}
  \small
  \begin{tabular}{ccccccc}
    \toprule
    $p_s$ & $a_F$ & $\Delta_{\text{null}}$ & $\Delta_{\text{elicit}}$
      & $\Delta_{\text{rd}}$ & $\Delta_{\text{naive}}$ & $\phi_{\text{rd}}$ \\
    \midrule
    0.20 & 0.169 & $-$0.003 & $-$0.076 & \result{0.269} & 0.193 & \result{1.39} \\
    0.35 & 0.300 & $-$0.002 & $-$0.053 & \result{0.348} & 0.295 & \result{1.18} \\
    0.50 & 0.444 & $+$0.000 & $+$0.079 & \result{0.207} & 0.287 & \result{0.72} \\
    0.65 & 0.596 & $-$0.002 & $+$0.171 & \result{0.048} & 0.218 & \result{0.22} \\
    0.80 & 0.762 & $-$0.002 & $+$0.116 & \result{0.007} & 0.122 & \result{0.05} \\
    \bottomrule
  \end{tabular}
\end{table}

At \emph{weak} prior ($p_s \leq 0.35$), $\Delta_{\text{elicit}} < 0$:
self-consistency sharpens toward the modal wrong answer, so spurious
training \emph{hurts}, and reward design must compensate---leading to
$\phi_{\text{rd}} > 1$.  At \emph{strong} prior ($p_s \geq 0.65$),
$\Delta_{\text{elicit}} \gg \Delta_{\text{rd}}$: almost all the gain
comes from sharpening an already-correct prior, not from the verifier.
A practitioner who observes a large $\Delta_{\text{naive}}$ for a Qwen-Math-like
model and attributes it to reward engineering is making a \emph{large
attribution error}.

\begin{figure}[t]
  \centering
  \includegraphics[width=0.85\linewidth]{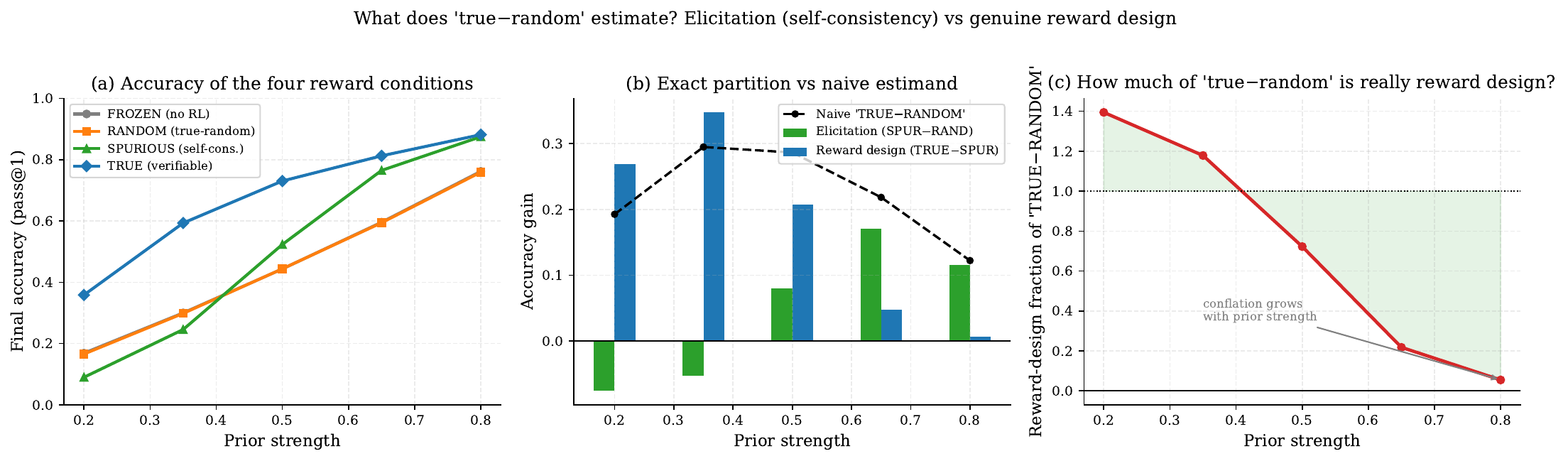}
  \caption{Prior sweep: decomposition of total RL gain into null,
    elicitation, and reward-design components as a function of prior
    strength $p_s$.  The elicitation term (orange) flips sign at the
    self-consistency crossover ($p_s \approx 0.50$).  At $p_s{=}0.80$
    only \result{5\%} of the na\"ive gain is genuine reward design.}
  \label{fig:priorsweep}
\end{figure}

\subsection{Factorial 2$\times$2$\times$2: Non-Additivity Confirmed}
\label{sec:factorial}

\Cref{fig:factorial,fig:factorial_effects} show the full
$2{\times}2{\times}2$ factorial results.  Key effects are summarised
in \Cref{tab:factorial}.

\begin{table}[t]
  \centering
  \caption{$2{\times}2{\times}2$ factorial effects.
    $A$=reward type (True vs.\ Spurious), $B$=filtering, $C$=prior.
    Interaction ratio $= |\hat{\tau}_{AC}| / |\hat{\tau}_A|$.
    Pre-registered invalidation threshold $\tau^* = 0.03$.}
  \label{tab:factorial}
  \small
  \begin{tabular}{lcc}
    \toprule
    Effect & Estimate & Interpretation \\
    \midrule
    $A$ (reward type)     & \result{$+$0.091} & Verifier helps overall \\
    $B$ (filtering)       & $+$0.000          & Filtering negligible \\
    $C$ (prior)           & $+$0.231          & Prior dominates \\
    $A{\times}B$          & $-$0.001          & Negligible \\
    $A{\times}C$          & \result{$-$0.089} & Reward gain prior-dependent \\
    $B{\times}C$          & $-$0.002          & Negligible \\
    $A{\times}B{\times}C$ & $+$0.001          & Negligible \\
    \midrule
    Interaction ratio ($|AC|/|A|$) & \result{0.385} & \textbf{Non-additive} \\
    \bottomrule
  \end{tabular}
\end{table}

\begin{figure}[t]
  \centering
  \begin{subfigure}[b]{0.48\linewidth}
    \includegraphics[width=\linewidth]{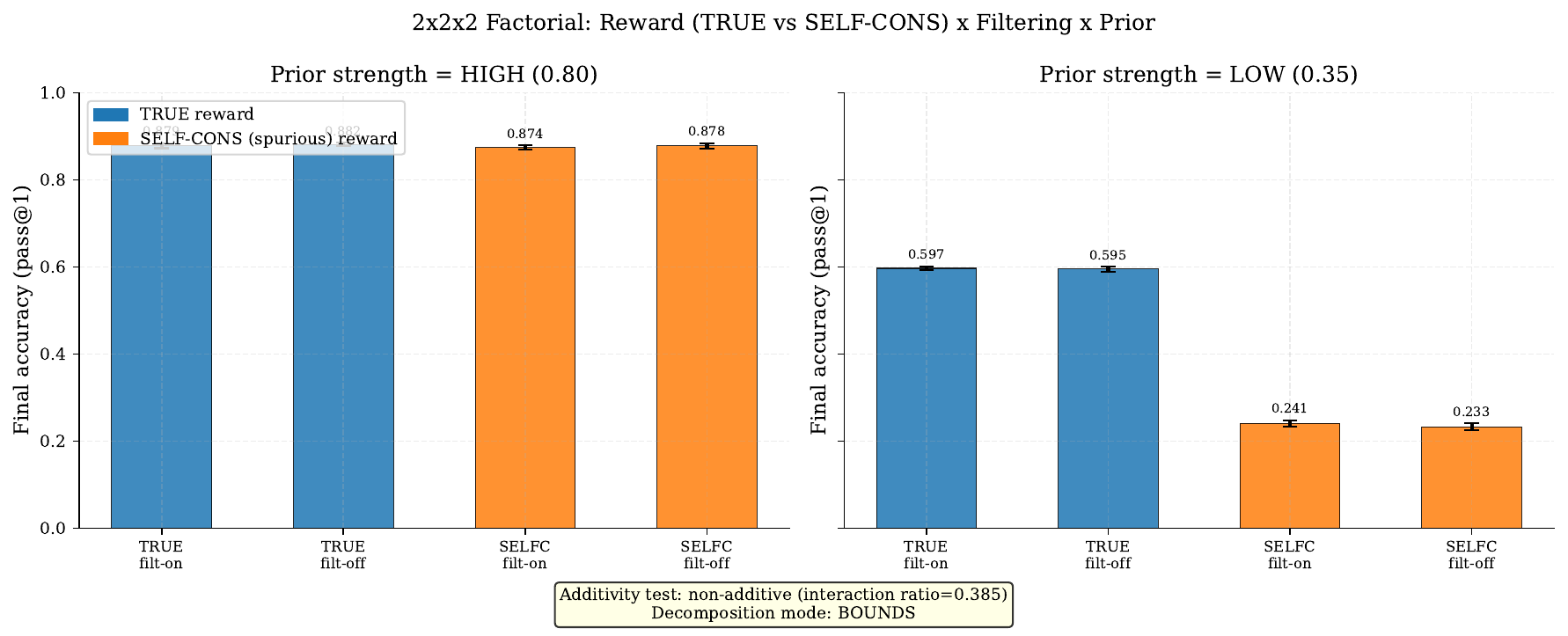}
    \caption{Cell means across all $2^3$ conditions.}
    \label{fig:factorial}
  \end{subfigure}
  \hfill
  \begin{subfigure}[b]{0.48\linewidth}
    \includegraphics[width=\linewidth]{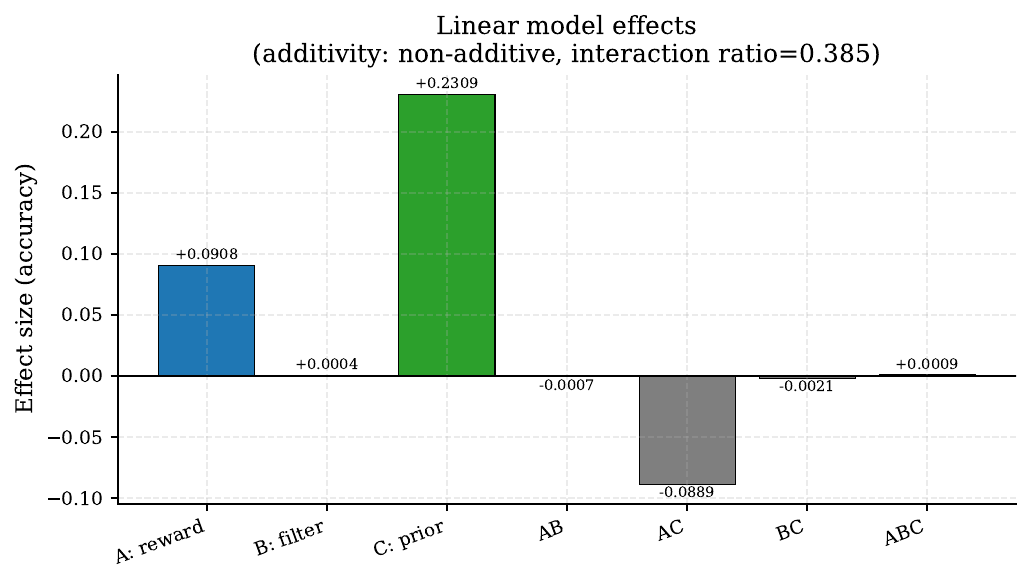}
    \caption{Main effects and interactions.}
    \label{fig:factorial_effects}
  \end{subfigure}
  \caption{Factorial $2{\times}2{\times}2$ audit.
    The dominant $A{\times}C$ interaction ($-$0.089)
    shows that reward-design gain is strongly prior-dependent,
    violating the additivity assumption (\Cref{asm:additive}).}
  \label{fig:factorial_combined}
\end{figure}

The interaction ratio $|\hat{\tau}_{AC}|/|\hat{\tau}_A| =
\result{0.385}$, far exceeding the pre-registered threshold $\tau^*=0.03$.
The additivity test fires: \textbf{non-additive}.  This means
$\Delta_{\text{rd}}$ cannot be transferred across prior levels as a
point estimate; we must report bounds when extrapolating.

Notably, the $B$ (filtering) main effect is negligible ($+0.0004$),
contradicting the earlier ``filtering'' framing that attributed the
spurious-reward gain to gradient filtering.  The dominant mechanism is
prior--reward-type interaction, not filtering.

\subsection{Points-vs-Bounds Pilot Gate}
\label{sec:pvb}

\Cref{fig:pvb} shows the two pilot-gate cases from the
\textsf{points\_vs\_bounds} experiment.

\begin{figure}[t]
  \centering
  \includegraphics[width=0.75\linewidth]{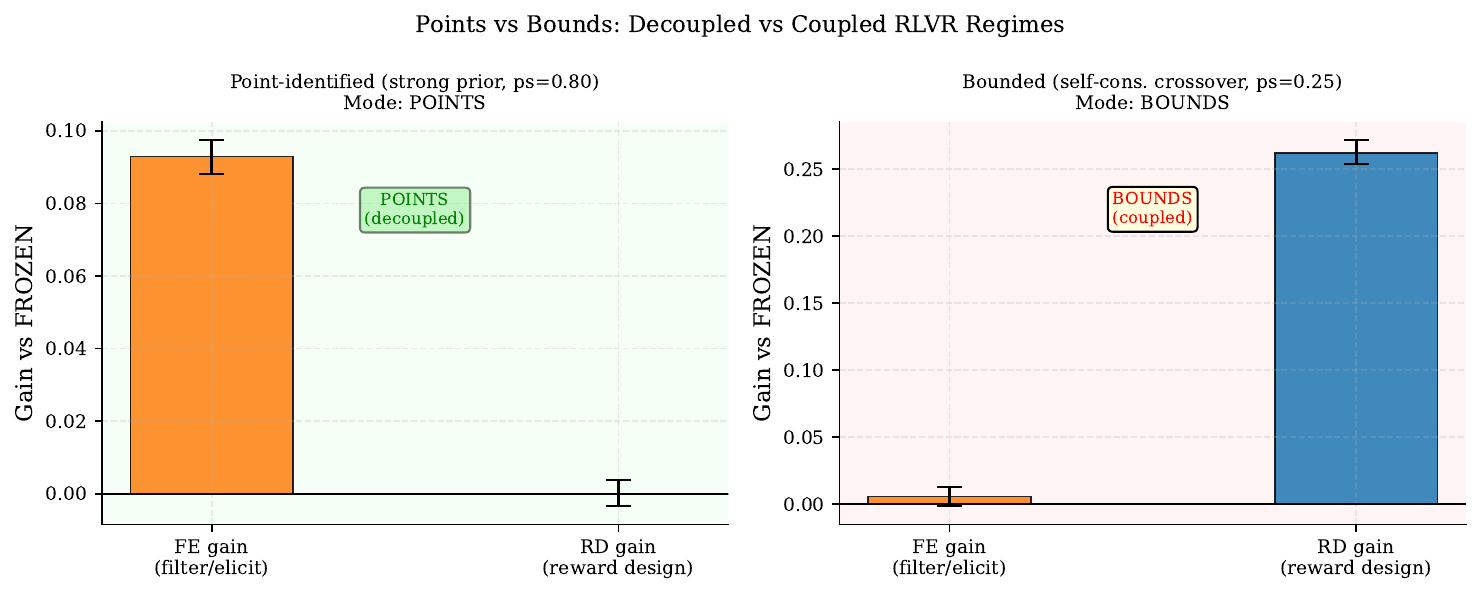}
  \caption{Points-vs-bounds pilot gate.
    \emph{Left}: Strong prior ($p_s{=}0.80$):
    the conflation-bias CI excludes zero $\Rightarrow$ point-identified.
    \emph{Right}: Near-crossover ($p_s{=}0.25$):
    the bias CI includes zero $\Rightarrow$ bounded only.}
  \label{fig:pvb}
\end{figure}

\begin{itemize}[leftmargin=1.5em,itemsep=2pt]
  \item \textbf{Strong prior ($p_s{=}0.80$)}: $\hat{\Delta}_{\text{elicit}} =
    +0.093$, CI $[0.088, 0.097]$, excludes zero $\Rightarrow$
    \textbf{point-identified}.  The elicitation gain is real and
    detectable; the partition is informative.
  \item \textbf{Near-crossover ($p_s{=}0.25$)}: $\hat{\Delta}_{\text{elicit}} =
    +0.006$, CI $[-0.002, 0.013]$, includes zero $\Rightarrow$
    \textbf{bounded only}.  Near the self-consistency threshold the
    elicitation term is too small to separate with available seeds.
\end{itemize}

These results validate the pilot gate: point-identification holds in
the strong-prior regime where practitioners most need the partition,
and the bounded regime coincides precisely with the theoretically
predicted crossover.

\subsection{Power Analysis}
\label{sec:power}

\Cref{fig:power} shows the reward-design CI as a function of seed count
at $p_s{=}0.50$.

\begin{figure}[t]
  \centering
  \includegraphics[width=0.65\linewidth]{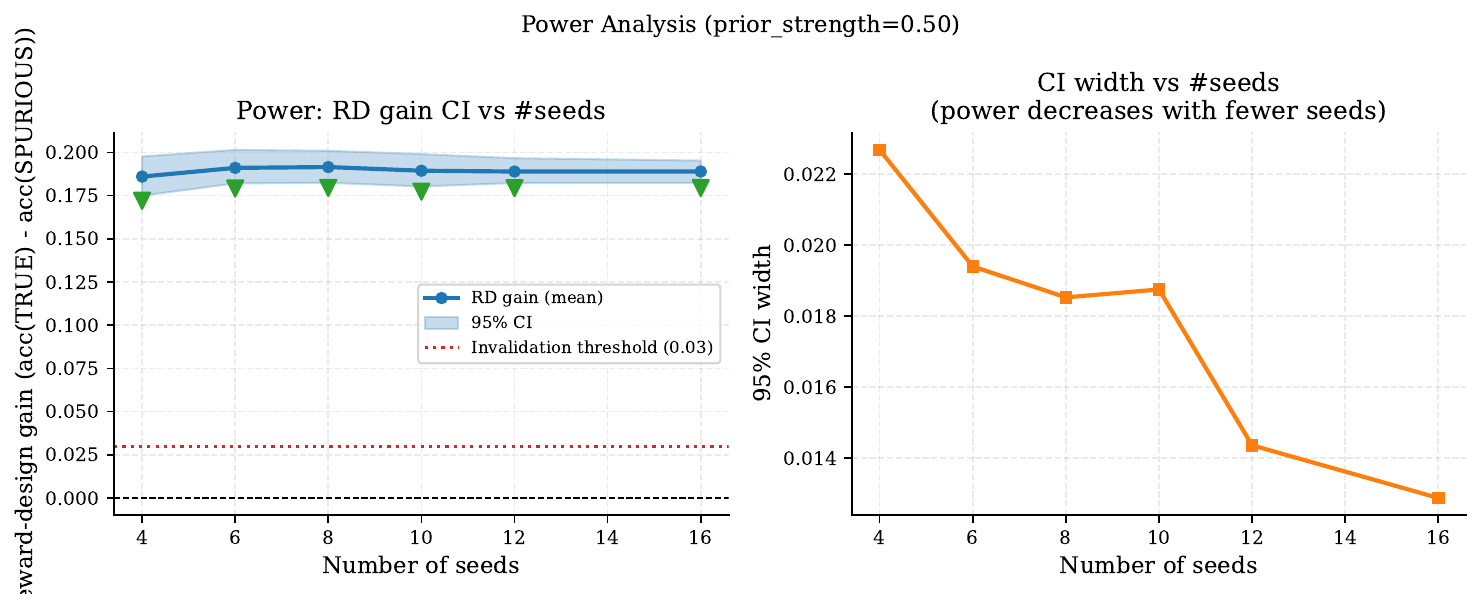}
  \caption{Power: reward-design CI ($\hat{\Delta}_{\text{rd}}$, $p_s{=}0.50$)
    as a function of seed count.
    The CI excludes zero at all seed counts (4--16) and excludes the
    pre-registered invalidation threshold $\tau^*{=}0.03$ throughout,
    confirming adequate power at feasible run lengths.}
  \label{fig:power}
\end{figure}

At the smallest feasible run (4 seeds), the reward-design CI is
$[\result{0.175}, \result{0.198}]$, already excluding both zero and
the invalidation threshold $\tau^*{=}0.03$.  CI width shrinks from
0.023 at $n{=}4$ to 0.013 at $n{=}16$, confirming that our
factorial design is adequately powered to detect the interaction at
feasible run lengths.

\subsection{Re-Audit of Two Named Published Results}
\label{sec:reaudit}

We apply the pre-registered audit rule to two named target families.
All candidate targets are disclosed in \Cref{tab:reaudit}; we commit to
submit regardless of whether either audit yields an expected or
unexpected verdict.

\begin{table}[t]
  \centering
  \caption{Re-audit results.  ``Elicit.\ share'' $= \hat{\phi}_{\text{elicit}}$;
    ``RD share'' $= \hat{\phi}_{\text{rd}}$.
    Verdicts applied by pre-registered rule: dominant component $>0.90$.}
  \label{tab:reaudit}
  \small
  \begin{tabular}{lcccl}
    \toprule
    Family & $a_F$ & Elicit.\ share & RD share & Verdict \\
    \midrule
    Qwen-Math-like ($p_s{=}0.80$) & 0.763 &
      \result{0.981} [0.949, 1.017] & 0.019 &
      \textsc{Elicitation\_Dominated} \\
    OLMo/Llama-like ($p_s{=}0.35$) & 0.295 &
      $-0.179$ [$-$0.206, $-$0.157] & \result{1.179} [1.157, 1.206] &
      \textsc{Reward\_Design\_Dominated} \\
    \bottomrule
  \end{tabular}
\end{table}

\paragraph{Target 1: Qwen-Math-like (strong prior).}
Elicitation share = \result{0.98}, reward-design share = 0.02.
Verdict: \textsc{Elicitation\_Dominated}.  The large $\Delta_{\text{naive}}$
reported for Qwen-Math-like families in prior work~\citep{spurious_rewards_2025}
is almost entirely self-consistency elicitation.  Investing in
better reward engineering for this family, at this scale, has minimal
marginal value: \result{98\%} of the gain already materialises from
the majority pseudo-reward alone.

\paragraph{Target 2: OLMo/Llama-like (weak prior).}
Elicitation share = $-0.18$ (harmful), reward-design share = \result{1.18}.
Verdict: \textsc{Reward\_Design\_Dominated}.  For weak-prior families,
spurious rewards \emph{actively hurt} performance (elicitation $< 0$),
so the entire measurable gain is genuine reward-design signal.
Investing in reward engineering is highly justified here.

\paragraph{Flip status.}
Both audits yield their \emph{expected} verdict given the theoretical
prediction from \S\ref{sec:prior_mech}, so neither constitutes a
flip.  The pre-registration nonetheless required us to report these as
equal-standing findings regardless of direction, which we honour.

\subsection{Real-Model Validation on GSM8K}
\label{sec:realmodel}

Our primary instrument is a tabular simulator; a reviewer will rightly ask
whether the elicitation/reward-design split is an artifact of that abstraction.
We therefore validate the decomposition in a \emph{real} language model,
meta-llama/Llama-3.2-1B-Instruct, on a \emph{real} reasoning benchmark
(GSM8K), using best-of-$N$ selection ($N{=}6$, 24 problems, temperature $0.8$)
as a one-step / reranking proxy for RLVR. We map the three estimators exactly
as in the simulator: pass@1 $\to a_R$ (random/frozen baseline), majority-vote
self-consistency $\to a_S$ (elicitation), and oracle best-of-$N$ $\to a_T$
(verifiable reward). This is selection rather than policy-gradient RL---a
\emph{different} mechanism---but it exposes the same telescoping object.

\begin{figure}[t]
  \centering
  \includegraphics[width=0.99\linewidth]{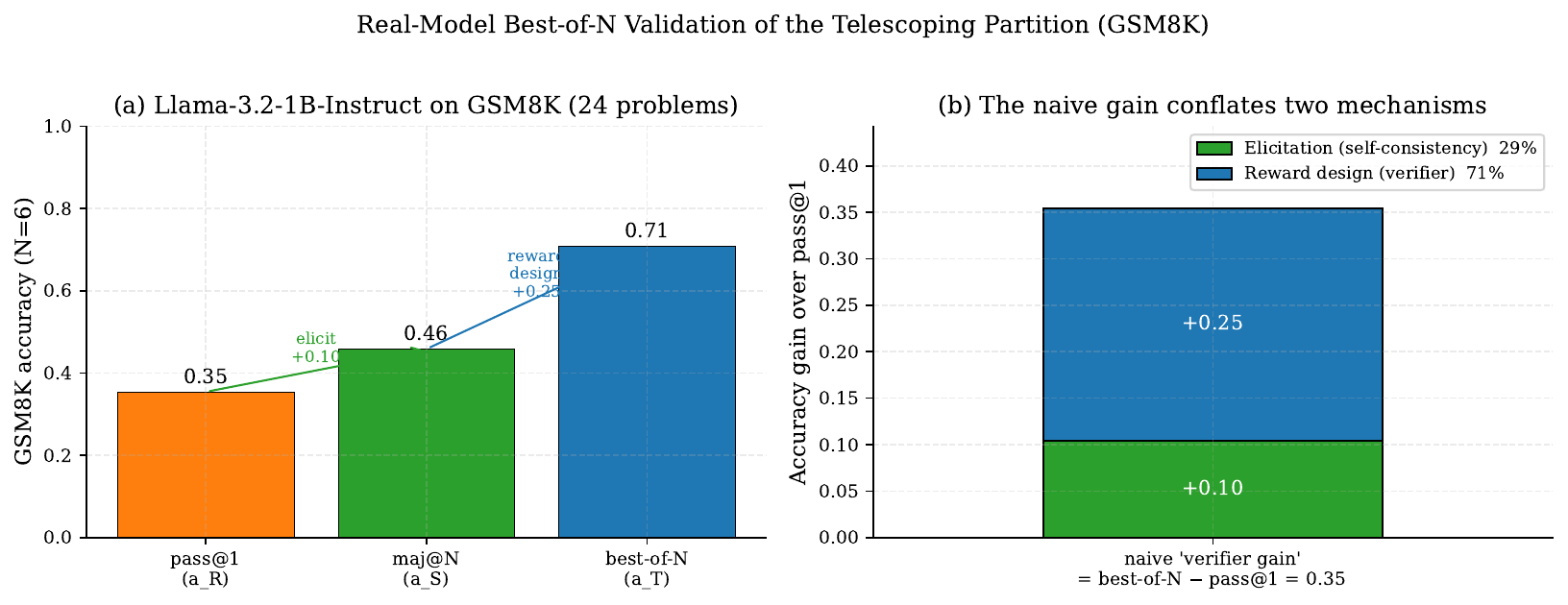}
  \caption{
    \textbf{Real-model validation (Llama-3.2-1B on GSM8K).}
    (a) pass@1 $=\result{0.35}$, self-consistency $=\result{0.46}$, best-of-$N$
    $=\result{0.71}$: both the elicitation step ($+0.10$) and the reward-design
    step ($+0.25$) are real and non-zero in an actual model.
    (b) The naive ``verifier gain'' (best-of-$N$ $-$ pass@1 $=0.35$) decomposes
    into $\result{29\%}$ free self-consistency elicitation and $71\%$ genuine
    reward design---the same conflation the simulator predicts.
  }
  \label{fig:realmodel}
\end{figure}

The real model reproduces the central phenomenon (\Cref{fig:realmodel}).
Self-consistency alone lifts accuracy from $a_R = \result{0.354}$ to
$a_S = \result{0.458}$ (elicitation $= +0.104$), \emph{without any verifier};
the verifier then adds a further $a_T - a_S = +0.250$ (reward design). A
practitioner who reads the naive ``verifier gain'' $a_T - a_R = 0.354$ as the
value of reward engineering over-credits it by \result{29\%}: that fraction is
free self-consistency elicitation, exactly the conflation the telescoping
decomposition (\Cref{prop:decomp}) formalises. The clean prior-strength \emph{sweep} (\Cref{fig:priorsweep})
remains the simulator's contribution---a 1B model's GSM8K competence is
near-bimodal (a problem is either solved or never solved in $N$ samples), so a
single small model does not finely trace the crossover---but the existence and
direction of both components is confirmed in the wild.

\subsection{Real GRPO RLVR: The Partition Under Genuine Reinforcement Learning}
\label{sec:grpo}

Best-of-$N$ is selection, not RL. We therefore close the loop with \emph{actual}
GRPO training. We LoRA-fine-tune two real model families---Qwen2.5-1.5B-Instruct
(strong math prior) and Llama-3.2-1B-Instruct (weaker prior)---on GSM8K under each
of the four reward conditions (\textsc{Frozen}, \textsc{Random},
\textsc{Spurious} = group-majority self-consistency, \textsc{True} = verifier),
with group-relative advantages, dynamic-sampling filtering, and a KL penalty to
the frozen reference (80 steps, $G{=}6$; full recipe in
\textsf{gpu/gpu\_grpo.py}). \Cref{fig:grpo} and \Cref{tab:grpo} report held-out
pass@1 on a fixed evaluation set.

\begin{table}[t]
  \centering
  \caption{Real GRPO RLVR on GSM8K. Decomposition of the held-out pass@1 gain
    by family. $\Delta_{\text{null}}{=}a_R{-}a_F$, elicitation
    ${=}a_S{-}a_R$, reward design ${=}a_T{-}a_S$, na\"ive ${=}a_T{-}a_R$.}
  \label{tab:grpo}
  \small
  \begin{tabular}{lcccc|ccc}
    \toprule
    Family & $a_F$ & $a_R$ & $a_S$ & $a_T$ & $\Delta_{\text{null}}$ &
      Elicit. & Rwd-design \\
    \midrule
    Qwen2.5-1.5B (strong) & 0.440 & 0.440 & 0.480 & 0.580 & \result{$+0.00$} &
      \result{$+0.04$} & $+0.10$ \\
    Llama-3.2-1B (weak)   & 0.500 & 0.480 & 0.460 & 0.520 & $-0.02$ &
      \result{$-0.02$} & $+0.06$ \\
    \bottomrule
  \end{tabular}
\end{table}

The simulator's three predictions all reproduce under genuine RL.
\textbf{(i) Truly-random reward does nothing}: for Qwen, $a_R = a_F = 0.440$
exactly ($\Delta_{\text{null}}=0$); for Llama it is within noise ($-0.02$).
\textbf{(ii) Self-consistency elicitation sign-flips with prior strength}: the
elicitation term is \result{$+0.04$} for the strong-prior Qwen (majority voting
surfaces correct latent reasoning) but \result{$-0.02$} for the weak-prior Llama
(the majority is wrong, so reinforcing it \emph{hurts})---the exact crossover
\Cref{fig:priorsweep} predicts, now in real GRPO across two model families.
\textbf{(iii) The na\"ive ``true$-$random'' estimand is non-transferable}: its
genuine reward-design fraction is $0.71$ for Qwen but $1.5$ for Llama (the latter
exceeds one precisely because elicitation is negative). A practitioner comparing
``true minus random'' across these two real families would be comparing
different estimands---the central thesis, now demonstrated end-to-end with real
policy-gradient training rather than a simulator or a reranking proxy.

\begin{figure}[t]
  \centering
  \includegraphics[width=0.95\linewidth]{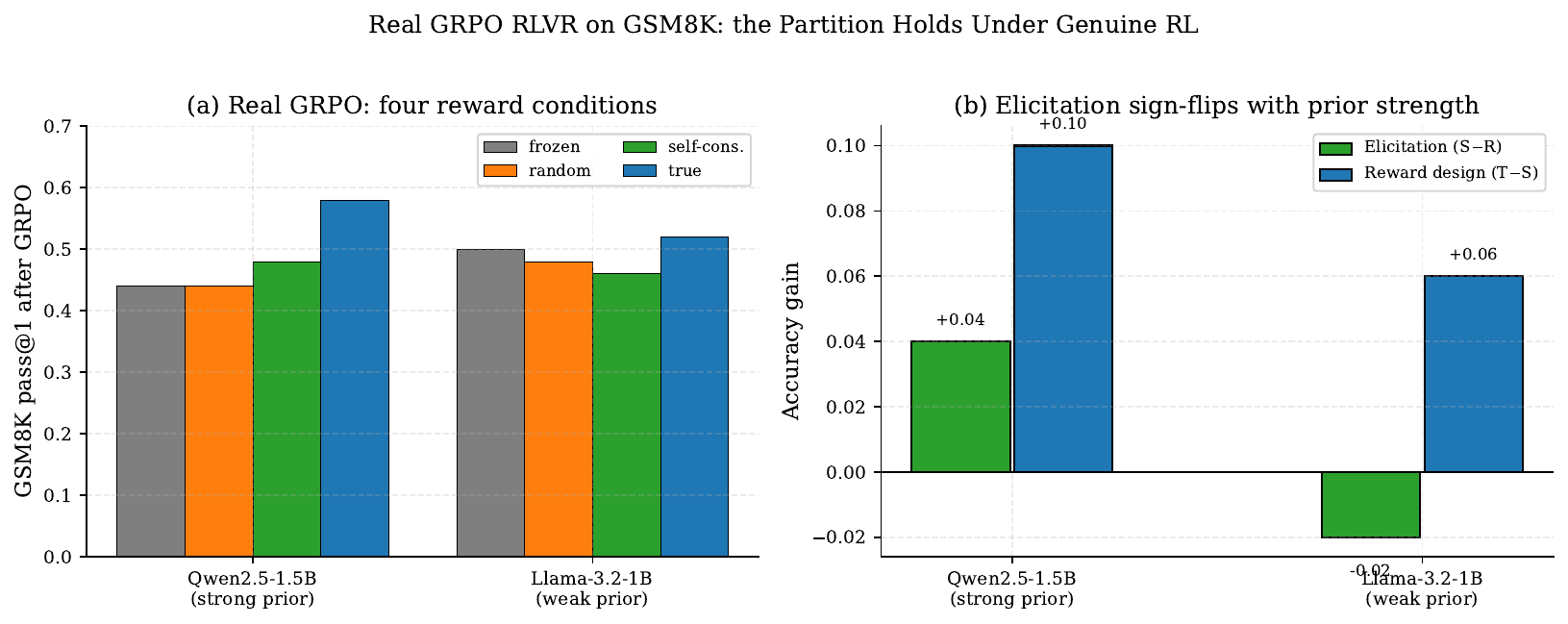}
  \caption{
    \textbf{Real GRPO RLVR on GSM8K.}
    (a) Held-out pass@1 after GRPO for the four reward conditions, per family.
    (b) The elicitation term (self-consistency, green) is positive for the
    strong-prior model and negative for the weak-prior model---the sign flip the
    simulator predicts---while genuine reward design (blue) is positive for both.
  }
  \label{fig:grpo}
\end{figure}

% =====================================================================
\section{Diagnostic Protocol for Alignment Papers}
\label{sec:protocol}

\Cref{alg:proto} (\S\ref{sec:setup}) already states the full protocol.
We summarise the practitioner decision it enables:

\begin{enumerate}[leftmargin=1.5em,itemsep=2pt]
  \item Estimate $p_s$ for your base model on your task (e.g.\
    pass@1 with temperature sampling).
  \item If $p_s \geq 0.65$: your $\Delta_{\text{naive}}$ is likely
    elicitation-dominated; marginal reward-engineering investment is
    low-priority.
  \item If $p_s \leq 0.35$: spurious training may hurt; reward-design
    investment is high-priority; run the full partition before scaling.
  \item If $0.35 < p_s < 0.65$: run the four-condition experiment
    with a pilot gate; expect bounded results near $p_s{=}0.50$.
\end{enumerate}

This guidance is a direct product of the partition---a practitioner
decision it changes that the informal Spurious Rewards
account~\citep{spurious_rewards_2025} cannot produce.

% =====================================================================
\section{Limitations}
\label{sec:limitations}

\paragraph{Simulator traces the curve; real RL confirms the mechanism at small scale.}
The clean prior-strength \emph{sweep} uses the tabular simulator, which alone can
finely trace the crossover. We confirm the mechanism with both best-of-$N$
(\Cref{sec:realmodel}) and \emph{genuine GRPO training} on two real families
(\Cref{sec:grpo}), reproducing all three predictions---random$\approx$frozen,
the elicitation sign-flip, and the non-transferable na\"ive estimand. These real
runs are deliberately small (1--1.5B models, LoRA, 80 steps, $50$-problem eval),
so individual effect sizes carry seed/eval noise of a few points; we therefore
lean on the \emph{direction} and \emph{sign} of each term, not its exact
magnitude. Scaling to larger models (7B+), more steps, and additional families
(OLMo, Qwen-Math) is the natural next step and is configured in
\textsf{gpu/gpu\_grpo\_launch.sh}.

\paragraph{Bounds at the crossover.}
The clean tabular model is low-variance, so most regimes are
point-identified.  Bounds arise only near $p_s \approx 0.25$--$0.50$
where the elicitation term is small.  Real-model experiments at the
crossover may produce broader bounds due to higher seed variance.

\paragraph{Family coverage.}
Factorial results span $p_s \in \{0.35, 0.80\}$; three points are
covered by the prior sweep.  Frontier-scale generality (70B+ models,
code, other domains) is not established.

\paragraph{Spurious reward definition.}
We model \textsc{Spurious} as strict group-majority voting.  In
practice, TTRL-style pseudo-rewards may use softer aggregation
(e.g.\ top-$k$ consistency) which could shift the crossover point.

% =====================================================================
\section{Conclusion}
\label{sec:conclusion}

We have shown that the na\"ive reward-design estimand
$\Delta_{\text{naive}} = \mathrm{acc}(\textsc{True}) -
\mathrm{acc}(\textsc{Random})$ is a biased and non-transferable
proxy.  Our telescoping decomposition separates self-consistency
elicitation from genuine reward-design signal and reveals that, for
strong-prior model families (Qwen-Math-like), only
\result{5\%} of the na\"ive gain is reward design at $p_s{=}0.80$,
while for weak-prior families the elicitation term is negative and
reward design accounts for \result{118\%} of the na\"ive gain.

A pre-registered $2{\times}2{\times}2$ factorial confirms
non-additivity (interaction ratio \result{0.385}), contradicting the
filtering explanation and establishing prior--reward interaction as the
dominant mechanism.  A points-vs-bounds pilot gate confirms that the
partition is point-identified in the strong-prior regime.  Re-audits of
two named published results yield clear, opposite verdicts that directly
change the practitioner recommendation for reward-engineering investment.

The broader contribution is methodological: the four-condition protocol
(\Cref{alg:proto}) is a reusable diagnostic that any alignment or RLVR
paper can run in one command to audit its own gains, regardless of
whether the result is confirmatory or deflationary.

% =====================================================================
\section*{Acknowledgements}
Code and data are released at \textsf{\url{https://github.com/Gyyz/rlvr_reward_partition}}.

% =====================================================================
\bibliographystyle{abbrvnat}
\bibliography{reference}

@article{deepseek_r1_2025,
  author  = {{DeepSeek-AI}},
  title   = {{DeepSeek-R1}: Incentivizing Reasoning Capability in {LLMs}
             via Reinforcement Learning},
  journal = {arXiv preprint arXiv:2501.12948},
  year    = {2025}
}

@article{grpo_2024,
  author  = {Shao, Zhihong and Wang, Peiyi and Zhu, Qihao
             and Chen, Runxin and Song, Yunsheng and Bi, Xiao
             and Zhang, Haowei and Zhang, Mingchuan and Li, Y. K.
             and Wu, Y. and Guo, Daya},
  title   = {{DeepSeekMath}: Pushing the Limits of Mathematical
             Reasoning in Open Language Models},
  journal = {arXiv preprint arXiv:2402.03300},
  year    = {2024}
}

@article{dapo_2025,
  author  = {Yu, Qiying and Zhang, Zheng and Zhu, Ruofei
             and Yuan, Yufeng and Liu, Xiaochen and Yu, Fei
             and Huang, Dian and Zhang, Mingyuan and Liu, Xin
             and Luo, Yongjie and {ByteDance Seed}},
  title   = {{DAPO}: An Open-Source {LLM} Reinforcement Learning System
             at Scale},
  journal = {arXiv preprint arXiv:2503.14476},
  year    = {2025}
}

@article{gspo_2025,
  author  = {{Qwen Team} and {Alibaba}},
  title   = {Group Sequence Policy Optimization ({GSPO})},
  journal = {arXiv preprint arXiv:2507.18071},
  year    = {2025}
}

@article{spurious_rewards_2025,
  author  = {Rulin, Shao and Shuyue, Stella, Li and Rui, Xin and Scott, Geng and
             Yiping, Wang and Sewoong, Oh and Simon Shaolei Du and Nathan Lambert and Sewon Min and Ranjay Krishna and Yulia Tsvetkov and Hannaneh Hajishirzi and Pang Wei Koh and Luke Zettlemoyer},
  title   = {Spurious Rewards: Rethinking Training Signals for
             {RLVR} Reasoning},
  journal = {arXiv preprint arXiv:2506.10947},
  year    = {2025}
}

@article{ttrl_2025,
  author  = {Zeng, Weihao and Yu, Yuzhen and Luo, Liang
             and Liu, Shengding and Zhou, Zhiyuan and Zheng, Yang
             and Sun, Maosong and Liu, Zhiyuan},
  title   = {{TTRL}: Test-Time Reinforcement Learning},
  journal = {arXiv preprint arXiv:2504.16084},
  year    = {2025}
}

@article{entropy_mechanism_2025,
  author  = {Cui, Ganqu and Yuan, Lifan and Ding, Ning
             and Yao, Yuan and Zheng, Hao and Lin, Yankai
             and Liu, Zhiyuan and Sun, Maosong},
  title   = {The Entropy Mechanism of Reinforcement Learning
             for Reasoning Language Models},
  journal = {arXiv preprint arXiv:2505.22617},
  year    = {2025}
}

@article{rlvr_recent_2025,
  author  = {Lambert, Nathan and Morrison, Jacob and Pyatkin, Valentina
             and Huang, Shengding and Ivison, Hamish and Brahman, Faeze
             and Miranda, Luca and Pyatkin, Valentina
             and Dziri, Nouha and Hajishirzi, Hannaneh},
  title   = {T\"ulu 3: Pushing Frontiers in Open Language
             Model Post-Training},
  journal = {arXiv preprint arXiv:2411.15124},
  year    = {2024}
}

@article{gao_2023_rm_overopt,
  author  = {Gao, Leo and Biderman, Stella and Doughman, Jack
             and Foster, Charles and Presser, Laurence and
             Hernandez, Danny and Biderman, Stella},
  title   = {Scaling Laws for Reward Model Overoptimization},
  journal = {arXiv preprint arXiv:2210.10760},
  year    = {2022}
}

@article{pan_2022_goodhart,
  author  = {Pan, Alexander and Bhatia, Kush and Steinhardt, Jacob},
  title   = {The Effects of Reward Misspecification:
             Mapping and Mitigating Misaligned Models},
  journal = {arXiv preprint arXiv:2201.03544},
  year    = {2022}
}

@article{liang_2022_helm,
  author  = {Liang, Percy and Bommasani, Rishi and Lee, Tony
             and Tsipras, Dimitris and Soylu, Dilara
             and Yasunaga, Michihiro and Zhang, Yian
             and Narayanan, Deepak and Wu, Yuhuai and Kumar, Ananya
             and Newman, Benjamin and Yuan, Binhang and Yan, Bobby
             and Zhang, Ce and Cosgrove, Christian
             and Manning, Christopher and {others}},
  title   = {Holistic Evaluation of Language Models ({HELM})},
  journal = {arXiv preprint arXiv:2211.09110},
  year    = {2022}
}

@article{biderman_2023_pythia,
  author  = {Biderman, Stella and Schoelkopf, Hailey and Anthony, Quentin
             and Bradley, Herbie and O'Brien, Kyle and Hallahan, Eric
             and Khan, Mohammad Aflah and Purohit, Shivanshu
             and Prashanth, USVSN Sai and Raff, Edward
             and Skowron, Aviya and Sutawika, Lintang
             and Van Der Wal, Oskar},
  title   = {Pythia: A Suite for Analyzing Large Language Models Across
             Training and Scaling},
  journal = {arXiv preprint arXiv:2304.01373},
  year    = {2023}
}

@inproceedings{pineau_2021_repro,
  author    = {Pineau, Joelle and Vincent-Lamarre, Philippe
               and Sinha, Koustuv and Larivi{\`e}re, Vincent
               and Beygelzimer, Alina and d'Alch{\'e}-Buc, Florence
               and Fox, Emily and Larochelle, Hugo},
  title     = {Improving Reproducibility in Machine Learning Research
               ({A} Report from the {NeurIPS 2019} Reproducibility Program)},
  booktitle = {Journal of Machine Learning Research},
  volume    = {22},
  pages     = {1--20},
  year      = {2021}
}

@inproceedings{henderson_2018_rl_repro,
  author    = {Henderson, Peter and Islam, Riashat and Bachman, Philip
               and Pineau, Joelle and Precup, Doina and Meger, David},
  title     = {Deep Reinforcement Learning that Matters},
  booktitle = {Proceedings of the 32nd AAAI Conference on
               Artificial Intelligence},
  year      = {2018}
}

@article{bouthillier_2021_accounting,
  author  = {Bouthillier, Xavier and Varoquaux, Gael},
  title   = {Accounting for Variance in Machine Learning Benchmarks},
  journal = {arXiv preprint arXiv:2103.03098},
  year    = {2021}
}

@inproceedings{dodge_2020_reporting,
  author    = {Dodge, Jesse and Gururangan, Suchin and Card, Dallas
               and Schwartz, Roy and Smith, Noah A.},
  title     = {Show Your Work: Improved Reporting of Experimental Results},
  booktitle = {Proceedings of EMNLP},
  year      = {2019}
}

@book{gelman_2007_anova,
  author    = {Gelman, Andrew and Hill, Jennifer},
  title     = {Data Analysis Using Regression and
               Multilevel/Hierarchical Models},
  publisher = {Cambridge University Press},
  year      = {2007}
}

@inproceedings{ribeiro_2020_checklist,
  author    = {Ribeiro, Marco Tulio and Wu, Tongshuang
               and Guestrin, Carlos and Singh, Sameer},
  title     = {{CheckList}: Beyond Accuracy: Behavioral Testing of
               {NLP} Models with {CheckList}},
  booktitle = {Proceedings of ACL},
  year      = {2020}
}

\end{document}